\documentclass[envcountsect]{bjmc}   % include options [ ... ] if needed

\usepackage{amsmath, amsthm, amsfonts}

\usepackage{graphicx}
\usepackage{epsfig}
\usepackage{pdfpages}
\usepackage{authblk}

\newtheorem{thm}{Theorem}[section]

\newtheorem{prop}[thm]{Proposition}

\theoremstyle{definition}

\theoremstyle{remark}
\newtheorem{rem}[thm]{Remark}

\doi{https://doi.org/10.22364/bjmc.2022.10.1.01}
\proceedings{10}{1}{1}
\begin{document}

%
%% TITLE:
%% Insert your title here using the \title command
\title{A PCA-based Data Prediction Method%\thanks{Notes about the article that
%%               should go on the front page may be placed
%%               here. General acknowledgments should be
%%               placed  at the end of the article.}
}
\titlerunning{A PCA-based Data Prediction Method}  % abbreviated title (for running head)
%
%
%% AUTHOR(s):
%%
%% The first name {other initials are optional and may be
%% inserted if this is the usual way of writing your name)
%% is followed by the surname, which should be typeset in
%% uppercase.
%%
%% === if one author:
%%

\authorrunning{Author}   % abbreviated name without initials
%            % (for running head)
%%
%% === if several authors from the same institute:
%%
%\author{First AUTHOR, Second BUTHOR %, ect.
% }
%%               \thanks{Grants or other notes
%%               about the author that should go on the front
%%               page may be associated with each author (as
%%               above). General acknowledgments should be
%%               placed at the end of the article.}
%\authorrunning{Author and Buthor}
%%               abbreviated author list (for running head)
%%               or, if more than two authors,
%\authorrunning{Author et al.}
%%
\author{
Peteris DAUGULIS\inst{1}, Vija VAGALE\inst{1}, Emiliano MANCINI\inst{2,3}, Filippo CASTIGLIONE\inst{4}
}
%% === if several authors from various institutes:
%%
%%            indicate the number of the institute from the list
%%            of institutes below; each author may have more than
%%            one institute number:
%\author{First AUTHOR\inst{1},  Second BUTHOR\inst{2,3},
%Third CUTHOR\inst{2}%\thanks{Grants or other notes concerning an
%%           author that should go on the front page may be
%%            associated with each author. General acknowledgments
%%            should be placed at the end of the article.}
%            }
\authorrunning{Daugulis et al.}
%%            abbreviated author list (for running head)
%%            or, if only two authors,
%\authorrunning{Author and Buthor}
%
%
% INSTITUTE(s) and ADRESS(es}:
%%
%% === if one institute:
%\institute{Princeton University, Princeton NJ 08544, USA%\thanks{
%%               Notes concerning the institute that should go
%%               on the front page may placed here. General
%%               acknowledgments should be placed at the end
%%               of the article.}
%}
%%
\institute{Daugavpils University, Daugavpils, Latvia
\and
Data Science Institute, Hasselt University, Diepenbeek, Belgium
\and
Department of Global Health, Amsterdam UMC, Amsterdam, The Netherlands
\and
Institute for Computing Applications, Rome, Italy}
%% === if several institutes:
%% numbered automatically; to be separated by \and command.
%%
%\institute{name and address of the institute
%\and
%name and address of the next institute
%\and
%name and address of the next institute%\thanks{Grants or other
%%               notes about the institute that should go on the
%%               front page may be associated with each institute.
%%               General acknowledgments should be placed
%%               at the end of the article.}
%}
%
%% E_ADDRESSES:
%%

\dedication{0000-0003-3866-514X, 0000-0002-5428-6441, 0000-0002-5613-234X, 0000-0002-1442-3552}

\emails{peteris.daugulis@du.lv, vija.vagale@du.lv,emiliano.mancini@uhasselt.be,  filippo.castiglione@cnr.it }
%%               the list of authors' email addresses ordered in
%%               accordance with the order of authors in \author
%%
%
%% DEDICATION:
%%
%\dedication{text}   % optional
%

\maketitle      % typesets the title of the contribution

%% BODY OF YOUR PAPER:
%%
\begin{abstract}
The problem of choosing appropriate values for missing data is
often encountered in the data science. We describe a novel method
containing both traditional mathematics and machine learning
elements for prediction (imputation) of missing data. This method
is based on the notion of distance between shifted linear
subspaces representing the existing data and candidate sets. The
existing data set is represented by the subspace spanned by its
first principal components. Solutions for the case of the
Euclidean metric are given.
%\keywords{PCA, prediction, imputation, subspace, distance}
\end{abstract}
%
%%Your text comes here.
%

\section{Introduction}

\subsection{Outline}

In this article we describe a method of predicting unknown values
of variables which is based on PCA and metric (Euclidean or other)
in the ambient real linear space of variables -  \sl the
PCA-distance method\rm. Our motivation and goal is the problem of recovering 
missing data assuming that the set of complete data samples is
approximated by the hyperplane spanned by the first principal
components and the set of candidate points form another shifted
subspace. Regression is not used in this method. In the case of the Euclidean metric we give exact
solutions which use orthogonal projections and extrema of
quadratic functions. We prove all the included mathematical statements. The computation algorithm and the activity
diagram of the PCA-distance method are given. It is assumed that
arithmetical operations of various data entries are justified,
i.e. all data units are dimensionless. Steps of the method can be
interpreted in terms of machine learning.

%
%Here goes the text.
%\begin{equation}\label{eq:area}
%  S = \pi r^2
%\end{equation}
%One can refer to equations like this: see equation (\ref{eq:area}). One can also
%refer to sections in the same way: see section \ref{sec:nothing}. Or
%to the bibliography like this: \cite{Cd94}.

\subsection{Background and previous work}\label{sec:nothing}

\paragraph{Prediction as a mathematical modelling problem.}

In most sciences and research areas it is necessary to generate
(predict, impute, estimate) missing information (data, relations,
rules etc.) if partial initial information is given. In this
article we use the term \sl prediction\rm\ to denote such methods.
 Such methods usually are based on minimizing errors and
finding extremal values of functions and discrete objects. Most
prediction methods start with building a suitable mathematical
object - a \sl model,\rm\ a discrete or continuous subset of an
ambient space, to represent the most important properties of the
given body of information (e.g. a discrete data set). Models may
mean representation in at least two senses - the evolution of a
system or a process, or the simplified description of an existing
system or a collection of data. Apart from a model of existing
data we must alse have a set of candidate values for predictions.

A range of general purpose methods which can also be used for
prediction purposes, such as approximation, interpolation,
extrapolation and others, have been developed 
\cite{Me}, (Mittal, \cite{M}). Machine learning approaches are
used (Bengio, Courville et al., \cite{BCV}), (Bzdok, Altman et al.,
\cite{BAK}).

The simplest prediction methods involve using mean values of
specified components of suitable data points, (Little and Rubin, 
\cite{LR}). There are prediction methods using least-square (linear
regression) ideas, (Bu, Dysvik et al., \cite{BDJ}). There are methods
assuming that the completely defined data samples belong to a
mutivariate distribution, (van Buuren, \cite{B}). In such methods
(Expectation-Maximization methods) parameters of the distribution
corresponding to the completely defined samples and missing values
are computed to maximize the likelyhood function. A popular
direction is based on the K-nearest neighbour (K-NN) idea which uses
a metric or a similarity measure in certain subspaces of the ambient
space, (Jonsson and Wohlin, \cite{JW}). In this approach missing
values are defined as means of corressponding values of nearest
completely defined data points. Nearness is defined using
Euclidean-like metrics in the subspace having dimensions where
values are defined for all data points.  See (Bertsimas, Pawlowski et
al., \cite{BPZ}).

\paragraph{Principal Components Analysis.}

An effective and widely used tool of data modelling and analysis is
the Principal Component Analysis (PCA), (Pearson, \cite{P}), (Eckart
and Young, \cite{EY}), (Hestenes, \cite{He}), (Hotelling, \cite{Ho1}),
(Hotelling, \cite{Ho2}). It is used to represent a discrete set of
data points as a shifted linear subspace which shows the most
important variables and their linear combinations. PCA is used for
linearization of data, dimensionality reduction, filtering out noise
and finding the most important linear combinations of data variables
(Meglen, \cite{Meg}), (Gorban, Kegl et al., \cite{GKWZ}).

PCA is a mathematical procedure that transforms the basis of the
space (i.e., a change of variable) which includes the set of data
we are interested in. It is mainly used to reduce the
dimensionality of a large data set. In fact, after the
transformation, the new coordinates of the basis (i.e., the new
variables) are ranked in terms of the ability to embed most of the
variability of the data set. Thus, focusing on few variables and
neglecting the others, one can keep most of the information
contained in the data set and focus on that.

In other words, to preserve as much variability as possible one
finds new variables that are linear functions of those in the
original dataset, that have the property of successively
maximizing the variance and, at the same time, are uncorrelated
with each other. The mathematical operations to perform a PCA are
based on the eigendecomposition of the data covariance matrix
(hence the principal components are eigenvectors of the data's
covariance matrix) also known as singular value decomposition of
the data matrix.

PCA is mainly a statistical tool developed by statisticians that has
found a large number of applications in many fields of science and
technology and that is currently used as a step to perform
predictions with machine learning methodologies. There is a data
prediction method - Bayesian Principal Component analysis, which
uses PCA (Oba, Sato et al., \cite{OS}).

\section{Main results}

In this section we describe the minimal distance idea and prove the mathematical results for the Euclidean case.

\subsection{The minimal distance idea}

\paragraph{Notations and basic facts.}\label{sec:nothing2}

Given two subsets $\mathcal{A},\mathcal{B}$ of a metric space
$(\mathbb{M},d(\cdot,\cdot))$ we denote by
$d(\mathcal{A},\mathcal{B})$ the distance between $\mathcal{A}$ and
$\mathcal{B}$: $d(\mathcal{A},\mathcal{B})=\inf\limits_{a\in
\mathcal{A},b\in \mathcal{B}}d(a,b)$. If $\mathcal{U}, \mathcal{V}$
are shifted linear subspaces in a real Euclidean (inner-product)
space $E$ then $d(\mathcal{U},\mathcal{V})=\min\limits_{u\in
\mathcal{U},v\in \mathcal{V}}d(u,v)$, a nonnegative real number. We
consider $\mathbb{R}^{n}$ as the Euclidean space with the norm $||x
||=\sqrt{x^{T}x}$ and the metric $d(x,y)=||x-y||$, for consistency
its elements are defined as columns. We denote by
$proj_{\mathcal{V}}$ the orthogonal projection onto $\mathcal{V}\le
\mathbb{R}^{m}$. We denote the subspace spanned by the columns of a
matrix $V$ by $\mathcal{V}$ (using $\backslash mathcal$ letters) or
$\langle V \rangle$.

\paragraph{Predicting one variable.}

First we describe the problem we are trying to solve
in the case of predicting one variable. Suppose we have a system
described by $m-1$ independent variables (indicators)
$x_{1},x_{2},...,x_{m-1}$ and one dependent variable $y$. We consider
$\mathbb{R}^{m}$ with some metric, for example, the Euclidean
metric.

Suppose we have a set of complete measurements
$\mathcal{S}=\{(x_{11},...,x_{1,m-1},y_{1}),...,$ $...,(x_{s1},...x_{s,m-1},y_{m})\}$, $|\mathcal{S}|=s$. We also have an incomplete measurement - a sequence of values $X_{0}=(x_{01},...,x_{0,m-1})$ for which we want to
find a $y$-value which would be most appropriate in a rigorous
sense. It means finding a point on the line
$\mathcal{L}=\{x_{1}=x_{01},...,x_{m-1}=x_{0,m-1}\}$ (\sl the prediction
line \rm) which is special with respect to $\mathcal{S}$.

In order to extract the most important property of  $\mathcal{S}$ we
choose "the first term of approximation" - the linear approximation,
which is sufficient for most meaningful predictions dealing with
missing data coming from various sources with possibly different
standards and formats. For this purpose we can use PCA. Arrange the
coordinates of $\mathcal{S}$-elements as a matrix $S=\left[
                               \begin{array}{c|c|c|c}
                                 x_{11} & ... & x_{1,m-1} & y_{1}\\
                                 \hline
                                 ... & ... & ... & ...\\
                                 \hline
                                 x_{s1} & ... & x_{s,m-1} & y_{m}\\
                               \end{array}
                             \right]
$. Choose $n$ (first) principal components (PC) of $S$. Construct
the subspace $\mathcal{P}$ (\sl shifted principal subspace\rm)
spanned by these PC. It is a shifted linear subspace of dimension
$n$ in $\mathbb{R}^{m}$. By construction $\mathcal{P}$ is spanned
by a basis (the principal components) which diagonalizes the
covariance matrix of data with maximal diagonal elements, in every
dimension there is the shift by the column average.

We have that $\dim(\mathcal{P})=r$, $\dim(\mathcal{L})=1$. Suppose
that $\mathcal{L}\not\subseteq \mathcal{P}$. If $\mathcal{P}$ and
$\mathcal{L}$ intersect (in one point $(x_{01},...,x_{0,m-1},y_{0})$)
then take the intersection point as the prediction - the predicted
$y$-value of the sequence $X_{0}$ is $y_{0}$. This  See Fig.1 for
the case $n=3, r=2$. We note that this case is essentially the least
square prediction.

\begin{center}

\includegraphics[height=45mm,width=60mm]{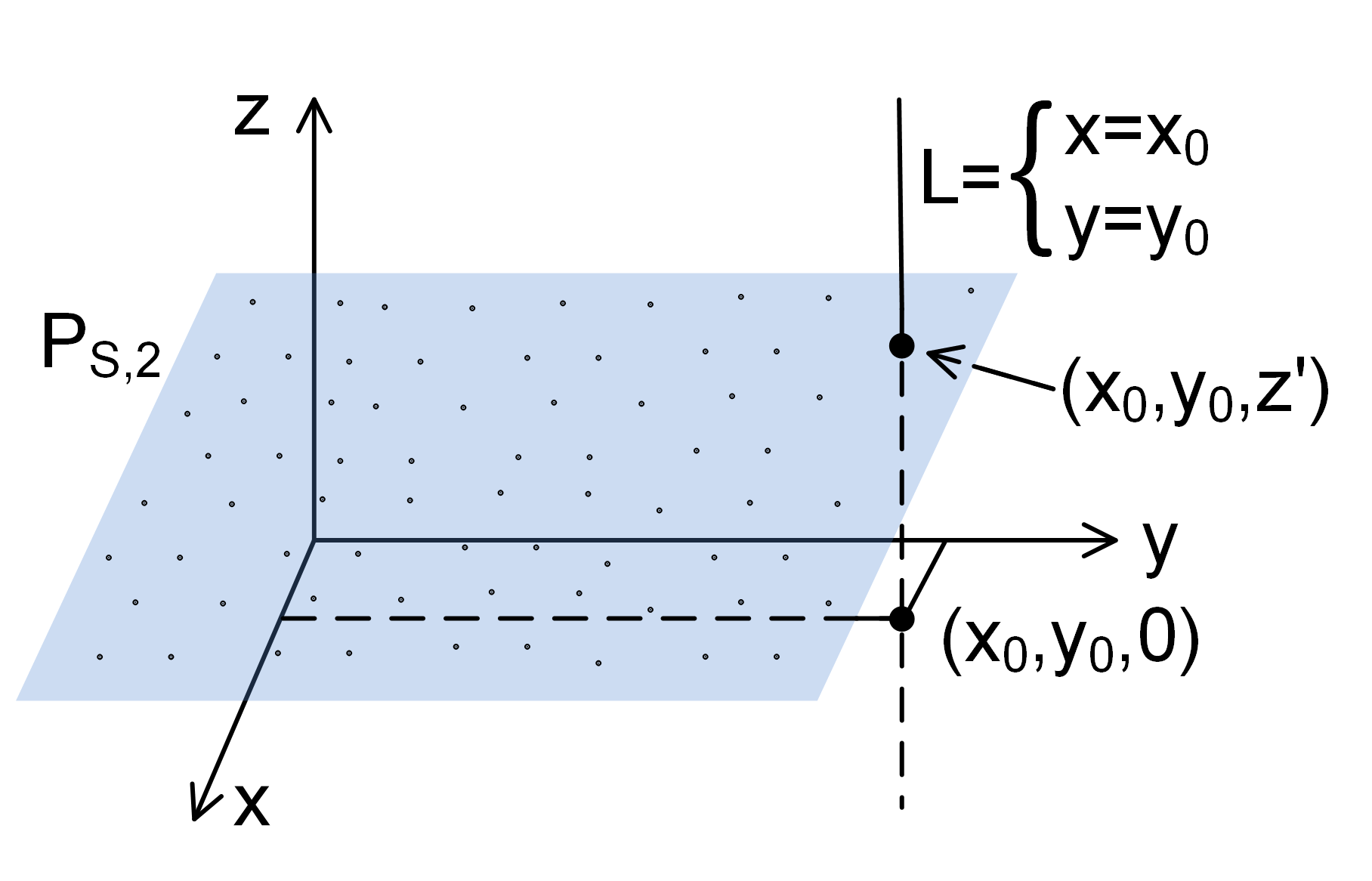}

Fig.1.  The case $r=2$ in $\mathbb{R}^{3}$.
\end{center}

Consider the case when $\mathcal{P}$ and $\mathcal{L}$ do not
intersect. Our proposal is to choose a point in $\mathcal{L}$
minimizing the distance to $\mathcal{P}$ as our prediction. At least
one such point exists. We want to find $l_{0}\in \mathcal{L}$ for
which there is $p_{0}\in \mathcal{P}$ such that
$d(l_{0},p_{0})=\min\limits_{l\in \mathcal{L},p\in
\mathcal{P}}d(l,p)=d(\mathcal{L},\mathcal{P})$, $l_{0}$ gives the
desired "predicted" $y$-value.

\begin{center}

\includegraphics[height=45mm,width=60mm]{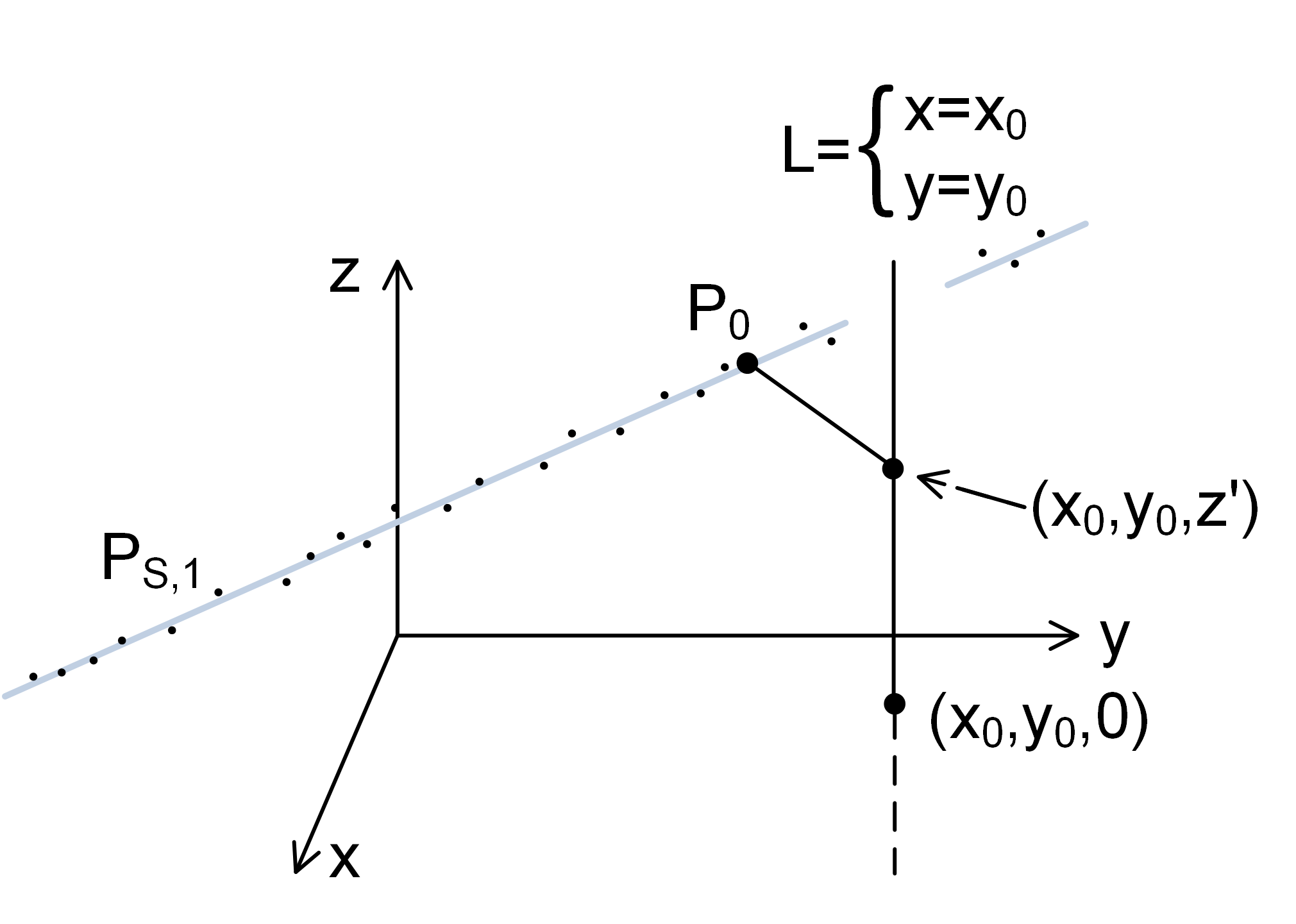}

Fig.2.  The case $r=1$ in $\mathbb{R}^{3}$.
\end{center}

If we use the Euclidean distance then
$p_{0}=proj_{\mathcal{P}}(l_{0})$.

\paragraph{Predicting more than one variable.}

We can have a situation where more than one entry of data points are
missing - we need to predict more than one variable for each
measured data sample. We work in $\mathbb{R}^{n}$ and have a
sequence $X_{0}=(x_{01},...,x_{0k})$, $k<n-1$, for which we want to
find the missing $n-k$ values. In this case the shifted linear
subspace $\mathcal{L}=\{x_{1}=x_{01},...,x_{k}=x_{0k}\}$ (\sl the
prediction space\rm ) has dimension $k$. Again we can use the
minimal distance idea: find a $l_{0}\in \mathcal{L}$ for which the
minimal distance to $\mathcal{P}$ is achieved.

We note that in all cases $\mathcal{L}\cap \mathcal{P}\ne \emptyset$
is equivalent to $d(\mathcal{L},\mathcal{P})=0$.

%In each case we may have more than one point on $\mathcal{L}$
%achieving the minimal distance to $\mathcal{P}$.

Prediction by minimizing distances can also be generalized for the
cases when data or candidate data models are nonlinear varieties in
ambient spaces.

\subsection{Solutions for Euclidean spaces}

\paragraph{One dimensional prediction space - prediction
line.}\label{14}

In this section we describe exact solutions for a prediction line
$\mathcal{L}$, $\dim(\mathcal{L})=1$, and an arbitrary principal
subspace $\mathcal{P}$ in case of the Euclidean metric. These
solutions are based on orthogonal projections and extrema of
quadratic functions.

The first preposition deals with the case of a linearly independent
generating set of the subspace $\mathcal{P}$ with respect to which
we find the special point on the prediction line $\mathcal{L}$.

\begin{prop}\label{1} Let $p_{1},...,p_{n}$ be linearly independent elements in $\mathbb{R}^{m}$
, $P=[p_{1}|...|p_{n}]$ is the $m\times n$ matrix obtained by
joining $p_{1},..,p_{n}$. Denote

\begin{equation}\label{9}
W=P(P^{T}P)^{-1}P^{T}-E_{m}=[w_{1}|W'],
\end{equation}
 where $w_{1}$ is the first
column of $W$. Let $L=\{\left[
  \begin{array}{c}
    t \\
    \hline
    l' \\
  \end{array}
\right]|t\in \mathbb{R}\}$, $l'\in \mathbb{R}^{m-1}$ fixed, an
affine line in $\mathbb{R}^{m}$. Let $\mathcal{P}=\langle
p_{1},...,p_{n}\rangle\le \mathbb{R}^{m}$.

\begin{enumerate}

\item If $w_{1}=0$ then for any $l\in \mathcal{L}$ there is a point $p\in \mathcal{P}$
such that $d(l,p)=d(\mathcal{L},\mathcal{P})$.

\item If $w_{1}\neq 0$ then
$d(l_{pred},p_{0})=d(\mathcal{L},\mathcal{P})=\min\limits_{l\in
\mathcal{L},p\in \mathcal{P} }d(l,p)$ for $l_{pred}\in \mathcal{L}$
and $p_{0}\in \mathcal{P}$ if and only if

$l_{pred}=\left[
  \begin{array}{c}
    t_{pred} \\
    \hline
    l' \\
  \end{array}
\right]\in \mathbb{R}^{n}$ where

\begin{equation}\label{88}
t_{pred}=-\cfrac{1}{||w_{1}||^2}w^{T}_{1}W'l'.
\end{equation}

\end{enumerate}

\end{prop}

\begin{proof} Let $l=
\left[
  \begin{array}{c}
    t \\
    \hline
    l' \\
  \end{array}
\right]\in \mathcal{L}$. $P^{T}P$ is invertible since columns of $P$
are linearly independent.   It is known that
$proj_{\mathcal{P}}=P(P^{T}P)^{-1}P^{T}$, (Meyer, \cite{Mey}).
Furthermore,
$d(l,\mathcal{P})=||proj_{\mathcal{P}}(l)-l||=||(P(P^{T}P)^{-1}P^{T}-E_{m})l||=|||Wl||$.

We express $Wl$ as the linear combination of $W$-columns:

\begin{equation}\label{103}
Wl=tw_{1}+W'l'.
\end{equation}

\begin{enumerate}

\item Let $w_{1}=0$. Then for any $l\in \mathcal{L}$ $d(l,\mathcal{P})=||W'l'||$, it does
not depend on $l$,
$d(l,proj_{\mathcal{P}}(l))=d(\mathcal{L},\mathcal{P})$.

\item Let $w_{1}\neq 0$.
%$W'l'=proj_{\langle w_{1}\rangle}(W'l')+proj_{\langle
%w_{1}\rangle^{\perp}}(W'l')$.
We use the fact $||x||^2=||proj_{V}(x)||^2+||proj_{V^{\perp}}(x)||^2$ (the generalized Pythagorean theorem). Taking $x=Wl$ and $V=\langle w_{1}\rangle$ we get

\begin{multline}\label{104}
||Wl||^2=||proj_{\langle w_{1}\rangle}(Wl)||^{2}+||proj_{\langle
w_{1}\rangle^{\perp}}(Wl)||^2=\\
=||tw_{1}+proj_{\langle w_{1}
\rangle}(W'l')||^2+||proj_{\langle w_{1}
\rangle^{\perp}}(W'l')||^2.
\end{multline}

$l$ such that $||Wl||$ is minimal will be achieved for the unique
$t$ satisfying $tw_{1}=-proj_{\langle w_{1}\rangle}(W'l')$, hence

\begin{equation}\label{104}
t_{pred}=-\cfrac{1}{||w_{1}||^2}w^{T}_{1}W'l'.
\end{equation}

\end{enumerate}
\end{proof}

\begin{rem} Note that in any case $d(l,p)=d(l,\mathcal{P})$, $l\in \mathcal{L}$,
$p\in \mathcal{P}$, if and only if $p=proj_{\mathcal{P}}(l)$.
\end{rem}

\begin{rem} If $(p_{1},...,p_{n})$ is an (ordered) orthonormal basis of $\mathcal{P}$
then $P^{T}P=E_{n}$ therefore $proj_{\mathcal{P}}=W=PP^{T}$.
Ortonormality of $(p_{1},...,p_{n})$ takes place if
$p_{1},...,p_{n}$ are principal components.

\end{rem}

\begin{rem} If $(p_{1},...,p_{n})$ is a (not necessarily orthonormal, ordered) basis of $\mathcal{P}$
then it may be computationally more efficient to compute
$P(P^{T}P)^{-1}P$ via the $QR$ factorization of $P$. See \cite{Mey}.
The $QR$ factorization is suitable for matrices with large condition
number, it can be made computationally stable using Householder or
Givens reductions. If $P=QR$ where columns of $Q$ are orthonormal
and $R$ is a triangular matrix with positive diagonal entries then
$P(P^{T}P)^{-1}P^{T}=QQ^{T}$.

\end{rem}

The next proposition deals with the case when the generators of the
subspace $\mathcal{P}$ are not linearly independent.

\begin{prop}\label{2} Let $p_{1},...,p_{n}$ be elements of $\mathbb{R}^{m}$
, $P=[p_{1}|...|p_{n}]$ is the $m\times n$ matrix obtained by
joining $p_{1},..,p_{n}$. Let $PC=\left[
                                                       \begin{array}{c|c}
                                                        P_{e}  & O_{m,n-r} \\
                                                       \end{array}
                                                     \right]$
be such that $rank(P_{e})=rank(P)$ (for example, a column echelon
form of $P$), $C$ is a $n\times n$ matrix of elementary column
operations.

Denote

\begin{equation}\label{13} W=P_{e}(P_{e}^{T}P_{e})^{-1}P_{e}^{T}-E_{m}=[w_{1}|W'],
\end{equation}

\noindent where $w_{1}$ is the first column of $W$. In these
notations the statements of Proposition \ref{1} are true.

\end{prop}

\begin{proof} Columns of $P_{e}$ form a basis
for $\mathcal{P}$. Denote the columns of $P_{e}$ by
$p'_{1},...,p'_{r}$, $\mathcal{P}_{e}=\langle
p'_{1},...,p'_{r}\rangle$. Then
$proj_{\mathcal{P}}=proj_{\mathcal{P}_{e}}=P_{e}(P_{e}^{T}P_{e})^{-1}P_{e}^{T}$.
Repeat the proof of Proposition \ref{1} substituting $P$ by $P_{e}$.

\end{proof}

The final proposition of this section gives another interpretation
and the same solution of the problem using the fact that the square
of the distance between a point on a line and a subspace is
quadratic function of a line parameter.

\begin{prop}\label{4} Let $p_{1},...,p_{n}$ be elements in $\mathbb{R}^{m}$, $\mathcal{P}=\langle p_{1},...,p_{n}\rangle\le \mathbb{R}^{m}$. Let
$\mathcal{L}=\{t\in \mathbb{R}\Big|\left[
  \begin{array}{c}
    t \\
    \hline
    l' \\
  \end{array}
\right]\in \mathbb{R}^{m}\}$, $l'\in \mathbb{R}^{m-1}$ fixed, an
affine line in $\mathbb{R}^{m}$.  Let $l_{i}=\left[
  \begin{array}{c}
    t_{i} \\
    \hline
    l' \\
  \end{array}
\right]\in L$, $i\in \{1,2,3\}$, $t_i$ distinct. Let
$proj_{\mathcal{P}}(x)=Wx$, $x\in \mathbb{R}^{m}$, with the first
column of $W$ being nonzero. Let
$||proj_{\mathcal{P}}(l_{i})-l_{i}||^2=d_{i}$. Then
$d(l_{pred},p_{0})=d(\mathcal{L},\mathcal{P})$ for $l_{pred}\in
\mathcal{L}$ and $p_{0}\in \mathcal{P}$ iff $l_{pred}=\left[
  \begin{array}{c}
    t_{pred} \\
    \hline
    l' \\
  \end{array}
\right]$ where $t_{pred}=-\cfrac{a_{1}}{2a_{2}}$ and
$[a_{0},a_{1},a_{2}]^{T}$ is the solution of the linear system

\begin{equation}\label{5}
 \left[
  \begin{array}{ccc}
    1 & t_{1} & t_{1}^{2} \\
    1 & t_{2} & t_{2}^{2} \\
    1 & t_{3} & t_{3}^2 \\
  \end{array}
\right]\cdot \left[
               \begin{array}{c}
                 a_{0} \\
                 a_{1} \\
                 a_{2} \\
               \end{array}
             \right]=
             \left[
               \begin{array}{c}
                 d_{1} \\
                 d_{2} \\
                 d_{3} \\
               \end{array}
             \right]
 \end{equation}

\end{prop}

\begin{proof} If $l=
\left[
  \begin{array}{c}
    t \\
    \hline
    l' \\
  \end{array}
\right] $ then $||proj_{\mathcal{P}}(l)-l||^2$ is a nonconstant
nonnegative quadratic function $a_{0}+a_{1}t+a_{2}t^2$ of $t$. Its
coefficients can be determined by considering it values at $3$
values of $t$, say, $t_{1}, t_{2}, t_{3}$, coefficients are
solutions of (\ref{5}). The minimum of
$||proj_{\mathcal{P}}(l)-l||^2$ is achieved when
$t=-\cfrac{a_{1}}{2a_{2}}$.

Alternatively, $Wl=tw_{1}+W'l'$, therefore

\begin{equation}\label{105}
||Wl||^2=(Wl)^{T}Wl=||w_{1}||^2t^2+2w^{T}_{1}W'l'\cdot
t+||W'l'||^2.
\end{equation}

The minimum of $||Wl||^2$ as a function of $t$ is
achieved when $t=-\cfrac{1}{||w_{1}||^2}w^{T}_{1}W'l'$.
\end{proof}

\begin{rem}
If we use a general inner product $(x,y)=x^{T}My$ inducing the norm
$||x||_{M}=\sqrt{x^{T}Mx}$ where $M$ is a symmetric positive
definite matrix then (\ref{88}) has to be substituted by another
formula

\begin{equation}\label{99}
t_{pred}=-\cfrac{1}{||w_{1}||_{M}^2}w^{T}_{1}MW'l'.
\end{equation}

\end{rem}

\begin{rem}
Other norms such as $||\cdot ||_{p}$, $1\le p$, or $||\cdot
||_{\infty}$ can be considered in a similar way. The distance from
$l\in \mathcal{L}$ to $\mathcal{P}$ can be found using the unit
circle of the norm.
\end{rem}

\paragraph{Multidimensional prediction space.}

This section contains results when $\dim(\mathcal{L})>1$. The first
proposition gives a general solution if $\mathcal{P}$ is spanned by
linearly independent generators.

\begin{prop}\label{6} Let $p_{1},...,p_{n}$ be linearly independent elements in $\mathbb{R}^{m}$
, $P=[p_{1}|...|p_{n}]$ the $m\times n$ matrix obtained by joining
$p_{1},..,p_{n}$. Let $\mathcal{P}=\langle p_{1},...,p_{n}\rangle\le
\mathbb{R}^{m}$. Let

\begin{equation}\label{10}
W=P(P^{T}P)^{-1}P^{T}-E_{m}=[W_{k}|W'],
\end{equation}
 where $W_{k}$ is the block
of the first $k$ columns of $W$. Let $\mathcal{L}=\{\left[
  \begin{array}{c}
    t \\
    \hline
    l' \\
  \end{array}
\right]\in \mathbb{R}^{m}\}$, $t= \left[
  \begin{array}{c}
    t_{1} \\
    ... \\
    t_{k} \\
  \end{array}
\right] \in \mathbb{R}^{k}$, $l'\in \mathbb{R}^{m-k}$ fixed, an
affine $k$-dimensional subspace in $\mathbb{R}^{m}$.

\begin{enumerate}

\item If $W_{k}=0$ then for any $l\in \mathcal{L}$ there is a point $p\in \mathcal{P}$
such that $d(l,p)=d(\mathcal{L},\mathcal{P})$.

\item If $W_{k}\neq 0$ then
$d(l_{pred},p_{0})=\min\limits_{l\in \mathcal{L},p\in
\mathcal{P}}d(l,p)$ for $l_{pred}\in \mathcal{L}$ and $p_{0}\in
\mathcal{P}$ if and only if $l_{pred}=\left[
  \begin{array}{c}
    t_{pred} \\
    \hline
    l' \\
  \end{array}
\right]$ where $t_{pred}\in \mathbb{R}^{k}$ is such that
\begin{equation}\label{11}
W_{k}t_{pred}=-proj_{\langle w_{r}\rangle}(W'l').
\end{equation}
If $rank(W_{k})=k$ then
\begin{equation}\label{12}
t_{pred}=-W_{k,L}W_{k}(W^{T}_{k}W_{k})^{-1}W^{T}_{k}W'v
\end{equation}
 where
$W_{k,L}$ is a left-inverse of $W_{k}$ ($W_{k,L}W_{k}=E_{k}$).
\end{enumerate}

\end{prop}

\begin{proof} Let $l=\left[
  \begin{array}{c}
    t \\
    \hline
    l' \\
  \end{array}
\right]\in \mathcal{L}$. Again we have that $P^{T}P$ is invertible and
$proj_{\mathcal{P}}=P(P^{T}P)^{-1}P^{T}$, 
$d(l,\mathcal{P})=||proj_{\mathcal{P}}(l)-l||=||(P(P^{T}P)^{-1}P^{T}-E_{m})l||=|||Wl||$. Again we express the product $Wl$ as a linear combinations of columns:

\begin{equation}\label{106}
Wl=W_{k}t+W'l'=W_{k}t+proj_{\langle
W_{k}\rangle}(W'l')+proj_{\langle W_{k}\rangle^{\perp}}(W'l').
\end{equation}

Using the orthogonality we have

\begin{equation}\label{107}
||Wl||^2=||W_{k}t+proj_{\langle
W_{k}\rangle}(W'l')||^2+||proj_{\langle
W_{k}\rangle\rangle^{\perp}}(W'l')||^2.
\end{equation}
\begin{enumerate}

\item Let $W_{k}=0$. Then for any $l\in \mathcal{L}$ $d(l,\mathcal{P})=||W'l'||$, it does
not depend on $l$.

\item Let $W_{r}\neq 0$. $||Wl||$ is minimal if and only if $$proj_{\langle W_{k}\rangle}(Wl)=W_{k}t+proj_{\langle
W_{k}\rangle}(W'l')=0.$$

$t$ can be chosen such that $proj_{\langle W_{k}\rangle}(Wl)=0$
since $proj_{\langle W_{k}\rangle}(W'l')$ is an element generated by
the columns of $W_{k}$ and it can be expressed in the form $W_{k}t$.
If $rank(W_{k})=k$ then $proj_{\langle W_{k}
\rangle}=W_{k}(W^{T}_{k}W_{k})^{-1}W^{T}_{k}$ and $W_{k,L}$ - the
left inverse of $W_{k}$, exists, thus $W_{k}t=-proj_{\langle
W_{k}\rangle}(W'l')$ implies $t$ is given by (\ref{12}).

\end{enumerate}

\end{proof}

\begin{rem} The condition 

\begin{equation}\label{107}
W_{k}t=-proj_{\langle W_{k}\rangle}(W'l')
\end{equation}

can be interpreted that $t$ is the coordinate column of
$-proj_{\langle W_{k}\rangle}(W'l')$ with respect to the sequence of
$W_{k}$-columns, a generating set of $\langle W_{k}\rangle$.

\end{rem}

\begin{prop}\label{8} Let $p_{1},...,p_{n}$ be elements in $\mathbb{R}^{m}$, $P=[p_{1}|...|p_{n}]$. Let $PC=\left[
                                                       \begin{array}{c|c}
                                                        P_{e}  & O_{m,n-r} \\
                                                       \end{array}
                                                     \right]$
be such that $rank(P_{e})=rank(P)$ (for example, a column echelon
form of $P$), $C$ is a $n\times n$ matrix of elementary column
operations.

Denote 

\begin{equation}\label{102}
W=P_{e}(P_{e}^{T}P_{e})^{-1}P_{e}^{T}-E_{m}=[W_{k}|W'],
\end{equation}

where $W_{k}$ is the block of the first $k$ columns of $W$. In these
notations the statements of Proposition \ref{6} are true.

\end{prop}

\begin{proof} See proof of Proposition \ref{2}.

\end{proof}

%\begin{prop}
%Let $v_{1},...,v_{n}$ be elements in $\mathbb{R}^{m}$ ,
%$V=[v_{1}|...|v_{n}]$ the $m\times n$ matrix obtained by joining
%$v_{1},..,v_{n}$. Let $L=\{t\in \mathbb{R}^{k}\Big|[t|v]^{T}\in
%\mathbb{R}^{m}\}$, $v\in \mathbb{R}^{m-k}$ fixed, an affine
%$k$-dimensional subspace in $\mathbb{R}^{m}$.
%\end{prop}
%
%\begin{proof}
%
%\end{proof}

The $\mathcal{L}$-subset with minimal distance to $\mathcal{P}$ can
also be found similarly to Proposition \ref{4}. Consider the
nonnegative quadratic surface in independent variables
$t_{1},...,t_{k}$ and the dependent variable $d$:
$||proj_{\mathcal{P}}(l)-l||^2=d$, find its coefficients using
scalar products or $\cfrac{(k+1)(k+2)}{2}$ points in general
position, use theory of quadratic forms or find partial derivatives
and solve the corresponding linear system. Details are given in the
proposition below. For other metrics the solution must be modified
accordingly.

\begin{prop}\label{3} Let all notations be as in Proposition \ref{6}. Let $w_{j}$ be the jth column of $W$.
Then $t_{pred}\in \mathbb{R}^{k}$ in the case 2. of Proposition \ref{6} is a
solution of the $k\times k$ linear system
\begin{equation}\label{7}
At=b,\text{where}\ [A]_{ij}=w^{T}_{i}w_{j},\ b_{i}=-w^{T}_{i}W'l'
\end{equation}

\end{prop}

\begin{proof} Let $l=\left[
  \begin{array}{c}
    t \\
    \hline
    l' \\
  \end{array}
\right]\in \mathcal{L}$, $t= \left[
  \begin{array}{c}
    t_{1} \\
    ... \\
    t_{k} \\
  \end{array}
\right] \in \mathbb{R}^{k}$. Again we interpret the product $Wl$  a linear combination of $W$-columns:

\begin{equation}\label{108}
Wl=\sum\limits_{i=1}^{k}t_{i}w_{i}+W'l'. 
\end{equation}
Using othogonality we get $||Wl||^2=$
$\sum\limits_{i,j}^{k}t_{i}t_{j}w^{T}_{i}w_{j}+2\sum\limits_{i=1}^{k}t_{i}w^{T}_{i}W'l'+||W'l'||^2$.
We have that 

\begin{equation}\label{101}
\cfrac{\partial ||Wl||^2}{\partial
t_{i}}=2\sum\limits_{j=1}^{k}t_{j} w^{T}_{i}w_{j}+2w^{T}_{i}W'l'.
\end{equation}

Equating it to $0$ for each $i$ we get the $k\times k$ linear
system of equations given in the statement.
\end{proof}

\paragraph{Removing outliers (the most influential points) - an extension of the Cook's distance idea to the PCA setting.}\label{1}

A desirable step in the process of finding hidden data features is
detection of outliers, (Zimek and Schubert, \cite{ZS}). Outliers
can be removed using the ideas of the Cook's distance,
(Cook, \cite{C}),(Kim, \cite{K}). The idea is for each data point $x$ to
compare projections of data points onto two principal component
hyperplanes - the PC hyperplane constructed with the whole data
set $\mathcal{S}$ and the PC hyperplane of the same dimension
constructed with the data set $\mathcal{S}\backslash x$. If the difference between these projections is relatively large, then $x$ is considered an outlier (or an unduly influential point) with respect to the construction of the PCA hyperplane. It is
related to the leave-one-out cross-validation.

In general, if we are given the coordinate column of a data point
$y$ and two projection matrices $H$ and $H'$ then $Hy-H'y=(H-H')y$ or its
norm represents the difference of the two projections. To estimate
the absolute difference between $H$ and $H'$ on the whole data
set, the sum over all data points of norms (squared) of such
projection differences $\sum\limits_{y\in
\mathcal{S}}||(H-H')y||^2$ is computed. For the relative difference
we divide it by the sum of norms of distances from data points to
their projections onto a chosen hyperplane, say,
$\sum\limits_{y\in \mathcal{S}}||(H-E_{m})y||^2$.

We explain it in more detail for the PCA setting. We use the
notations of Section \ref{14}. Let $P=[p_{1}|...|p_{n}]$ be the
$m\times n$ matrix where $p_{j}$ is the $j$th principal component
for the data matrix $S$. Let $P_{i}=[p_{i1}|...|p_{in}]$ be the
$m\times n$ matrix where $p_{ij}$ is the $j$th
 principal component for the data matrix $S\backslash S_{i*}$ ($i$th row
 removed). Construct the projection matrices
 $H=P(P^{T}P)^{-1}P^{T}$,
 $H_{i}=P_{i}(P_{i}^{T}P_{i})^{-1}P_{i}^{T}$, note that $H,H_{i}$ are $m\times m$
 matrices. Columns of $S^{T}=[x_{1}|...|x_{m}]$ are vectors of data points, columns of $(H-H_{i})S^{T}$ are differences of
 projections of data points. We use the Frobenius norm of matrices. We have that 
 
 \begin{equation}\label{100}
 \sum\limits_{x\in \mathcal{S}}||(H-H_{i})x||^2=||(H-H_{i})S^{T}||^2. 
 \end{equation}
 
 We can assume that 
 
 \begin{equation}\label{109}
 C_{i}=||(H-H_{i})S^{T}|| 
\end{equation} 
 measures the
 total (absolute) influence of the removal of the $i$th data point.
 Columns of $(H-E_{m})S^{T}$ are vectors orthogonal to the PC
 hyperplane having data points and their projection as endpoints, $||(H-E_{m})S^{T}||^2$
 is total sum of distance squares from data points to their
 projections onto the initial PC hyperplane.
 $RC_{i}=\cfrac{||(H-H_{i})S^{T}||}{||(H-E_{m})S^{T}||}$ can be chosen as
 relative influence or outlying measure of the data point $i$.

Clusters of  data points with large $RC$-value can be removed
after computing all $RC_{i}$ or iteratively.

\paragraph{Cross-validation issues.}

Cross-validation of the model can be done estimating in-sample and
out-of-sample mean square error (MSE) of predictions.
Leave-$p$-out and $k$-fold cross-validation can be used to split
the whole data set into the training and test subsets.

We explain in some detail the leave-$1$-out cross-validation for our method. We use the notations of the previous section. For each $i,1\le i\le |\mathcal{S}|$ we construct $P_{i}$ as described above (the projection to the PCA hyperplane of dimension $n$ which is computed removing the $i$-th data point), define $W_{i}=P_{i}(P^{T}_{i}P_{i})^{-1}P_{i}-E_{m}$. We proceed according to Proposition \ref{1}, use the formula (\ref{88}) and get the prediction $t_{pred}=y'_{i}$ for the $i$'th data point using the other data points as a training set. We can now compare vectors $[y_{1},...,y_{s}]$ and $[y'_{1},...,y'_{s}]$, find MSE etc.

\paragraph{Confidence interval estimation.}

Confidence intervals for this method can be estimated using the natural jacknife or bootstrap methods. Leave-$p$-out jacknife method can be used to generate a prediction distribution for a given initial data vector. Sufficient number of leave-$p$-out subsets of full data points and corresponding PCA hyperplanes are generated, predictions are computed for the given incomplete data vector. Another way to generate a prediction distribution is using bootstrapping by randomly choosing with replacement sufficiently large subsets of data points. Additionally, sufficient number of fictitious incomplete data vectors with a given distribution can be generated to estimate confidence intervals with a fixed full data matrix $S$.

\section{Implementation of the PCA-distance method}

\paragraph{Scaling.}

The data matrix $S$ can be scaled columnwise - each $s_{ij}\in S$ is
substituted by $\cfrac{s_{ij}-\overline{s_{*j}}}{\sigma_{j}}$ where
$\overline{s_{*j}}$ and $\sigma_{j}$ are the mean and the standard
deviation, respectively, of the $j$th column. If $\sigma_{j}=0$ then
the function $s_{ij}\mapsto s_{ij}-\overline{s_{*j}}$ is applied.
After the prediction is found the inverse transformation is
computed. We usually assume that the data is
scaled.

\subsection{An algorithm}

We describe the main steps for an algorithm implementing the
PCA-prediction method developed in Propositions \ref{1},\ref{2},
\ref{6}, \ref{8}, \ref{3}. In particular, $p_{i}\in \mathbb{R}^{m}$.
Notations of these propositions are used. See also Fig.3.

\begin{enumerate}

\item[Step 1] Identify vectors $p_{1},...,p_{n}$, form the $m\times n$ matrix
$P=[p_{1}|...|p_{n}]$. Go to Step 2.

\bigskip

\item[Step 2] Determine $rank(P)$. If $rank(P)=n$ (i.e. $p_{1},...,p_{n}$ are linearly
independent) then go to Step 3.1, otherwise go to Step 3.2.

\bigskip

\item[Step 3.1] Compute the $m\times m$  matrix
$W=P(P^{T}P)^{-1}P^{T}-E_{m}$. Go to Step 4.

\bigskip

\item[Step 3.2] Find a column-echelon form of $P$ - $[P_{e}|O_{m,n-r}]$.
Compute $W=P_{e}(P_{e}^{T}P_{e})^{-1}P_{e}^{T}-E_{M}$. Go to Step 4.
\bigskip

\item[Step 4] Identify the subspace $\mathcal{L}$. If $\dim(\mathcal{L})=1$
then go to Step 5.1, otherwise go to Step 5.2.
\bigskip
\item[Step 5.1] Subdivide $W=[w_{1}|W']$. Identify $l'$. Compute
$t_{pred}=-\cfrac{1}{||w_{1}||^2}w^{T}_{1}W'l'$. Go to Step 6.
\bigskip
\item[Step 5.2] Subdivide $W=[W_{k}|W']=[w_{1}|...|w_{k}|W']$. Identify $l'$. Compute the matrix elements
$a_{ij}=w^{T}_{i}w_{j}$, $b_{i}=-w^{T}_{i}W'l'$. Solve the linear
system (\ref{7}). If there are free unknowns then use additional
arguments to find a unique prediction $t_{pred}= \left[
  \begin{array}{c}
    t_{1} \\
    ... \\
    t_{k} \\
  \end{array}
\right] \in \mathbb{R}^{k}$. Go to Step 6.
\bigskip
\item[Step 6] Complete the
algorithm by returning $l_{pred}=\left[
  \begin{array}{c}
    t_{pred} \\
    \hline
    l' \\
  \end{array}
\right]$.

\end{enumerate}

\begin{center}

\includegraphics[width=100mm]{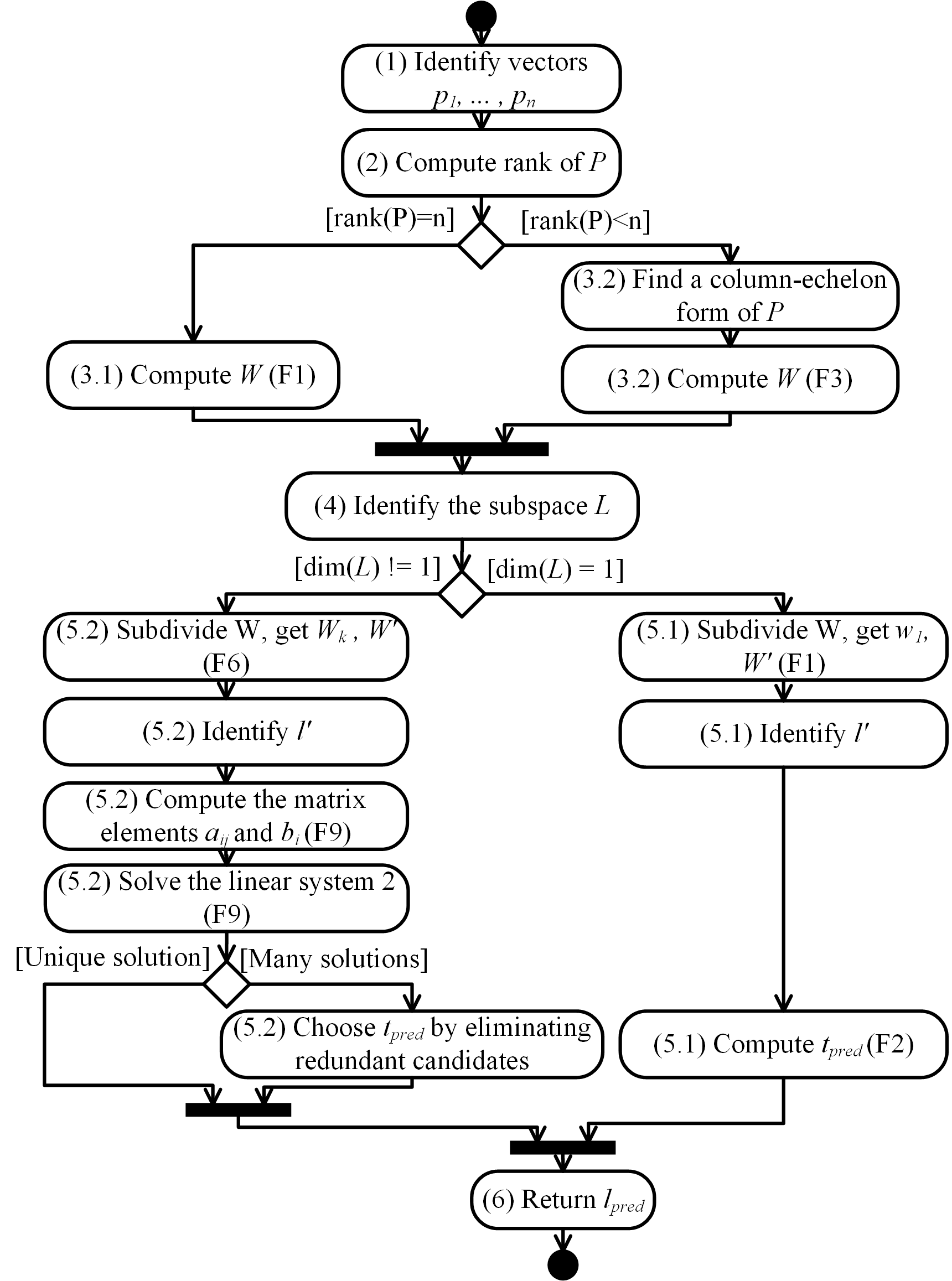}
 
Fig.3.  The activity diagram of the PCA-distance algorithm
\end{center}

\subsection{Implementation example}

The PCA-distance method have been implemented as a part of the research project dealing with predictions of drug-resistant pathogen strains for medical and pharmacological purposes. In our case the independent variables $x_{ij}$ are certain socio-economic indicators and $y_{i}$'s measure antimicrobial resistance of pathogens, see Acknowledgements.  Fig.4 shows one of the outcomes of this implementation - the world map coloured according to the PCA-distance predictions of the antimicrobial resistance. Typical size of data matrices (the matrix $S$) was about $300\times 7000$. Top $5\%$ of outliers were removed using the approach given in subsection \ref{1}. 

\begin{center}
\includegraphics[width=100mm]{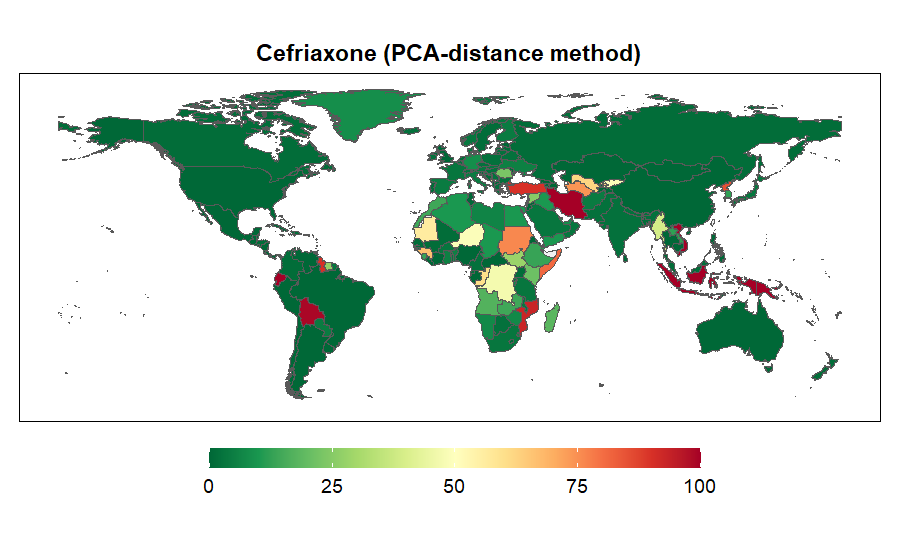}

Fig.4. Predictions of ceftriaxone-resistant pathogen percentage rates by the PCA-distance method

\end{center}

This approach has been compared with the one used in (Oldenkamp, Schultsz et al, \cite{OSMC}) - a betabinomial vector generalized linear model with a logit link function. For the considered cases prediction errors of this method are close to those obtained by the Oldenkamp-Schultsz-Mancini-Cappuccio method, see Fig.5. It must be noted that confidence intervals provided by the OSMC method contain predictions by the PCA-method in $20\%-30\%$ cases.

\medskip

\begin{center}
\includegraphics[width=100mm]{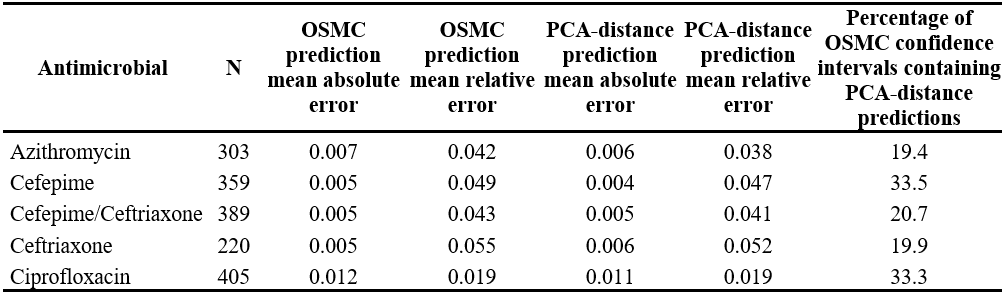}

\bigskip

Fig.5. Comparing OSMC and PCA-distance methods
\end{center}

\section{Discussion}

\paragraph{Machine learning features.}

The described prediction method may be interpreted as a technique
having features of unsupervised and semi-supervised machine
learning. The PCA-based dimensionality reduction which results in
the approximation of the initial data set by a low-dimensional
hyperplane, is a case of an unsupervised representation (feature)
learning which identifies the most important data indicators for
prediction purposes (Bengio, Courville et al., \cite{BCV}). Unsupervised PCA-based anomaly detection is
used to find outliers of the data set. In case the data matrix has
undefined elements, it can be filled by appropriate methods which
can be interpreted as cases of semisupervised learning.

Our prediction method is subject to typical machine learning
limitations and failures being caused by badly chosen assumptions,
conjectures, by data overfitting or underfitting.

\paragraph{Comparison of the PCA-distance method with other
prediction methods.}

The PCA-distance method seems to be more advanced and sensitive
compared to the naive mean value methods which do not use
linearization. One possible advantage of the mean value method is
that it always returns values within the interval specified by
existing measurements, e.g. it can not return a negative value if
all existing values are positive.

The PCA-distance method does not assume that data is distributed in
a special and uniform way therefore it seems more suitable to
process data having high dimensionality and data distributed in
different ways. Thus is it markedly different from methods in the
expectation-maximization family, (van Buuren, \cite{B}).

The PCA-distance method takes into account the whole data set of
complete samples. In this sense it is different form the K-NN
methods which take into account only a few complete samples,
(Bertsimas, Pawlowski et al., \cite{BPZ}). In our method there is no
need to arbitrarily specify the integer K. Taking into account the
linearization of the whole data set seems more justified for large
data sets with high dimensionality.  We note that the metric idea is
used in the space of all variables, including the dependent variables, not just the subspace of
variables which are defined for all samples.

Existing prediction methods using PCA seem to recover missing values
as coordinates of points on the subspace of principal components.
The PCA-distance method is different from such methods since it
finds extremal points on the candidate subspace with respect to the shifted subspace of principal components.

\section{Conclusion}

We offer a novel method for prediction of missing data which uses
only the linearization - the most important approximation idea, and
the notion of metric. The data set of full samples is linearized
using the PCA. The point or points on the candidate subspace having
the minimal distance to the linearized data subspace is chosen as
the prediction.   Closed formulas are obtained for the Euclidean
(canonical or generalized) metric case. Our method may be suitable
for data sets having high dimensionality and different data
distribution patterns for different variables since it does not explicitly assume any data distribution.

%\section{The first section}

%See BJMC class User guide for predefined environments and for provided additional possibilities.

%If you want to include a program code in your contribution, please typeset it using typewriter fonts (commands %\verb|\verb,\verb*| and environments \verb|verbatim,verbatim*|).

%A note for \textsc{Bib}\TeX\ users: do not forget to place the \verb|natbib.cfg| file in the working directory.

%\section{The second section}
%URLs are to be printed in the typewriter font. If you want to use  \verb|\url| command for formatting web addresses:
%\begin{verbatim}
%\url{http://www.gnu.org/philosophy/no-word-attachments.html}
%\end{verbatim}
%then load also the \verb|url| package, which avoids ugly line breaks and sticking out into a margin.

%% ACKNOWLEDGMENTS:     % (optional)
%%
\section*{Acknowledgements}

The authors acknowledge partial funding from the following national
funding agencies participating in the project  MAGIcIAN JPI-AMR
(https://www.magician-amr.eu/): State Education Development Agency
(VIAA, Latvia) and Italian Ministry of Education and Research (MIUR,
Italy).

%%         If you'd like to thank anyone, place your comments here.

%% REFERENCES:
%% Consult "Instructions for authors : About manuscript"
%% on the BJMC website) for references and citations.
%%
%% Since author-year system is used, only the year is recommended
%% as the optional argument with \bibitem command
%% if the bibliography is written by hand.
%%
%%
%\bibliographystyle{dcu1}       %For BibTeX users.
%                   %This command may be placeed also somewhere in the preamble.
%\bibliography{your bib. database(s)}   %For BibTeX users.
%\end{document}             %For BibTeX users.

%% The list below corresponds to the sample list of references
%%given in ``References and citations in BJMC''.

%% Author's information:  %(optional)
%%
%\section*{Author's information}  % or {Authors' information}
%%If you'd like to give some information, place your comments here
%%                     and remove the percent signs.
%

\received{January 19, 2021}{February 11, 2022}{February 16, 2022}
\end{document}